\documentclass[lettersize,journal]{IEEEtran}
\usepackage{amsmath,amsfonts}
\usepackage{amssymb}
\usepackage{amsthm}
\usepackage{algorithmic}
\usepackage{algorithm}
\usepackage{array}
\usepackage[caption=false,font=normalsize,labelfont=sf,textfont=sf]{subfig}
\usepackage{textcomp}
\usepackage{stfloats}
\usepackage{url}
\usepackage{verbatim}
\usepackage{graphicx}
\usepackage{cite}
\usepackage{booktabs}
\theoremstyle{remark}
\newtheorem{mylemma}{Lemma}

\begin{document}

\title{Reconfigurable Tendon-Driven Robots: Eliminating Inter-segmental Coupling via Independently Lockable Joints}
\author{Botao Lin, Shuang Song and Jiaole Wang
\thanks{This work was supported partly by National Key R\&D Program of China under Grant 2022YFB4703200, and partly by Talent Recruitment Project of Guangdong under Grant 2021QN02Y839, and in part by the Science Technology Innovation Committee of Shenzhen under Grant JCYJ20220818102408018 and GXWD20231129103418001.}
\thanks{B. Lin, S. Song, and J. Wang are with School of Mechanical Engineering and Automation, Harbin Institute of Technology (Shenzhen), Shenzhen, China.}
\thanks{$^*$Corresponding author: Jiaole Wang (wangjiaole@hit.edu.cn) and Shuang Song (songshuang@hit.edu.cn)}
}



\maketitle
\begin{abstract}
With a slender redundant body, the tendon-driven robot (TDR) has a large workspace and great maneuverability while working in complex environments.
TDR comprises multiple independently controlled robot segments, each with a set of driving tendons.
While increasing the number of robot segments enhances dexterity and expands the workspace, this structural expansion also introduces intensified inter-segmental coupling.
Therefore, achieving precise TDR control requires more complex models and additional motors.
This paper presents a reconfigurable tendon-driven robot (RTR) equipped with innovative lockable joints.
Each joint's state (locked/free) can be individually controlled through a pair of antagonistic tendons, and its structure eliminates the need for a continuous power supply to maintain the state.
Operators can selectively actuate the targeted robot segments, and this scheme fundamentally eliminates the inter-segmental coupling, thereby avoiding the requirement for complex coordinated control between segments.
The workspace of RTR has been simulated and compared with traditional TDRs' workspace, and RTR's advantages are further revealed.
The kinematics and statics models of the RTR have been derived and validation experiments have been conducted. 
Demonstrations have been performed using a seven-joint RTR prototype to show its reconfigurability and moving ability in complex environments with an actuator pack comprising only six motors.

\end{abstract}


\section{Introduction}
\IEEEPARstart{T}{endon-driven} robots (TDR) are widely used in many complex environments applications nowadays, including space station \cite{lastinger2020variable}, exploration of debris rescue \cite{wooten2018exploration}, minimally invasive surgery (MIS) \cite{dupont2022continuum,kim2019ferromagnetic,wang2019steering,song2021real,wang2021eccentric,kong2022dexterity}.
Since TDRs can be actuated remotely with actuators integrated out of the robot, TDRs have great maneuverability and large workspace while maintaining their slender size.
Typical TDRs consist of sequentially connected robot segments, with each segment having an independent set of control tendons.
Although increasing the number of robot segments can enhance TDR's flexibility and workspace, it inevitably introduces additional driving tendon sets and actuation motors \cite{endo2019super}.
Moreover, since the driving tendons of each segment pass through the segments in front of it, the increase in the driving tendon sets also leads to a rise in the coupling between robot segments during movements \cite{9420666,xu2018kinematics}, which necessitates more complex modeling and coordination of more motors for precise control of the TDR, such as learning-based control methods \cite{zhang2022survey,202200367ContinuumRobotsAnOverview}, increasing the control complexity.
Furthermore, the progressively expanding tendon set will prevent the robot from miniaturizing and adapting to some work scenarios with size requirements. 
Therefore, increasing the number of TDR controllable robot segments and reducing inter-segment coupling during movements is a long-standing challenging trade-off.

\begin{figure}[t]
\centering
\includegraphics[width=0.38\textwidth]{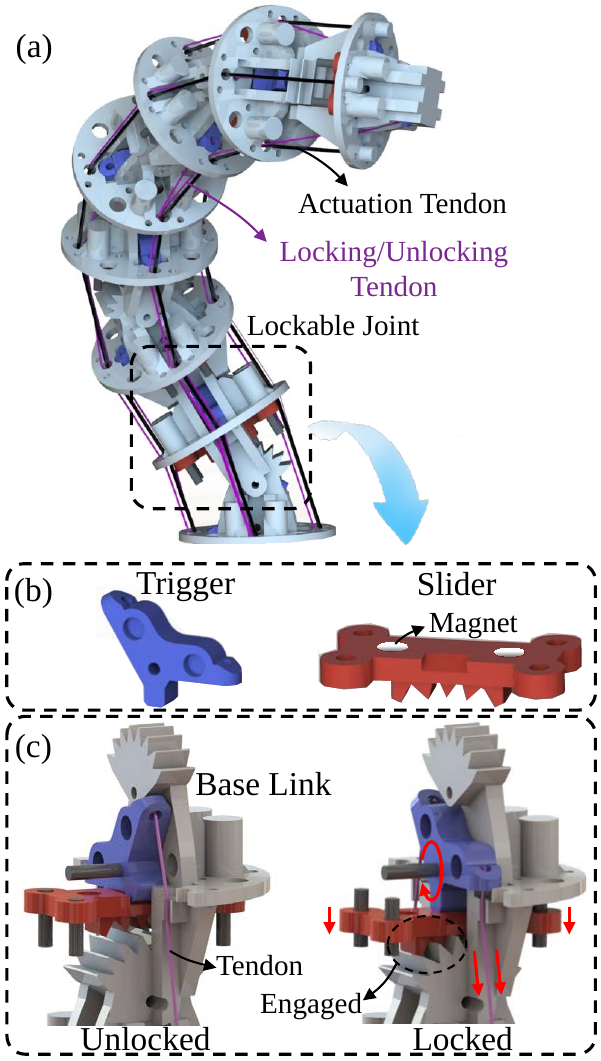}
\caption{Design of the RTR.
(a) The conceptual diagram of the RTR. 
The robot is composed of consecutively connected lockable joints.
One set of driving tendons passes through the whole robot, mounting on the distal joint.
Each lockable joint is connected with two tendons, used to switch the joint's state.
(b) The trigger and slider inside the lockable joint.
Two magnets are implanted in the slider.
(c) When the tendon pulls the trigger, it drives the slider downward, forming a dead-point configuration. 
In this state, the teeth on the slider engage with those on the below link, locking the joint stably.}
\label{fig:robot}
\end{figure}

In the past decade, researchers have proposed various methods for reducing the inter-segment coupling between TDR's segments and the control complexity of the robot, and have achieved many results.
Russo \textit{et al.} \cite{10606060} designed a helical tendon routing to decouple segments of TDR.
The helical tendon routing can maintain the length of the tendon unchanged when the backbone bending, benefiting both hardware and software.
Wang \textit{et al.} \cite{wang2019cable} proposed a hybrid modular cable routing method to decrease the total number of actuating tendons.
Some tendons are co-shared by adjacent joints, and the total number of actuating cables can be reduced to $n+1$ for an $n$ degree-of-freedom (DoF) TDR.
Jiang \textit{et al.} \cite{jiang2017new} designed an actuation pack with a tendon selector for the TDR.
The electromagnetic-clutch-based tendon selector realizes using one motor to drive different tendons steply in the procedure, reduces the number of actuators, and eases the complexity of the TDR control.
Because the control complexity of TDR originated from the strong coupling of driving tendons, reducing the complexity requires the reduction of tendons when adding movable segments, while the methods mentioned above have limited effect. 
%
%

Enabling TDRs to independently lock each robot segment is another emerging approach for reducing inter-segmental couping during TDR's movements.
During each step in the operation cycle, users only actuate the targeted segment to move and lock the shape of other segments, achieving inter-segmental decoupling.
Many locking methods have been investigated, including utilizing phase-changeable material, shape memory material, jamming, lockable mechanisms, etc., aiming for independent locking of robot segments.
Bishop \textit{et al.} \cite{bishop2022novel} proposed a shape-memory-alloy (SMA) based clutch for locking TDR.
Equipping the clutches, the TDR can achieve decoupled motions between robot segments by locking the targeted segments and actuating the rest parts. 
The work also demonstrated the partial locking function enlarges TDR's workspace significantly.
Wockenfuss \textit{et al.} \cite{wockenfuss2022design} employed granular jamming for locking the shape of TDR.
With each robot segment filled with granules, applying negative air pressure enables a sharp increase of the friction between granules and stiffens the targeted robot segment, then the rest part of TDR can deform separately.
Le \textit{et al.} \cite{le2020temperature} proposed a TDR with a phase-changeable PET tube inside each robot segment.
By heating/cooling the tube, users can regulate the stiffness of the robot segment and activate the targeted segment independently.  
Firouzeh \textit{et al.} \cite{firouzeh2015under} and Yan \textit{et al.} \cite{yan2021towards} implant shape memory materials inside TDR's joints.
by tuning the temperature with an embedded heater, each joint's locking state can be switched individually.



\begin{table}[t]
\centering
\caption{Comparison between underactuated TDR, overactuated TDR, and RTR.}
\begin{tabular}{cccc}
\toprule
Robot Type        & Control DoF & Workspace & \begin{tabular}[c]{@{}c@{}}Inter-segmental \\ Coupling\end{tabular} \\
\midrule
Underactuated TDR & Few         & Small     & Weak                                                               \\
Overactuated TDR  & Numerous    & Large     & Strong                                                             \\
RTR               & Few         & Maximum   & Weak\\
\bottomrule
\end{tabular}
\label{lab: benchmark}
\end{table}

Our previous work \cite{botao2022lockable} proposed a modular lockable joint has been proposed for TDR, and an early exploration of how the lockable joint can affect TDR's performance has been carried out.
In this paper, we propose a reconfigurable underactuated tendon-driven continuum robot (RTR) with novel mechanical lockable joints. 
Compared with the previous design, the novel lockable joint has a more reliable locking performance, and its locking/unlocking action can be actuated more conveniently and stably.
The RTR employs an innovative architecture featuring a compact tendon driving module that connects to the distal end-effector via only one set of driving tendons. 
Inter-segmental motion coupling issues observed in conventional TDRs caused by driving tendons are fundamentally eliminated in the RTR.
During motion execution, only target joints are activated while non-target joints remain at fixed angular positions via mechanical locking mechanisms. 
This time-phased actuation strategy allows the RTR to be easily controlled. 
Additionally, it enables users to dynamically reconfigure the robot by adjusting activated joint count and spatial distribution during each motion step for varying working scenarios.
A comparison of the key features of underactuated TDR, overactuated TDR, and RTR is presented in Table \ref{lab: benchmark}.
The main contributions of this paper can be listed as follows:
\begin{enumerate}
    \item A novel RTR with lockable joints is proposed, and the corresponding six-motor actuation pack is designed.
    \item The kinematics and statics of the RTR are modeled, a detailed analysis of the workspace and dexterity of RTR has been given, and the relationship between the robot's dexterity and configuration is shown.
    \item Experimental validations for verifying kinematics and statics models have been carried out, and the working performance of the RTR has been successfully demonstrated.
\end{enumerate}

\section{Design of RTR and Motion Strategy}

\subsection{Design of the RTR System}

\begin{figure*}
    \centering
    \includegraphics[width=0.9\textwidth]{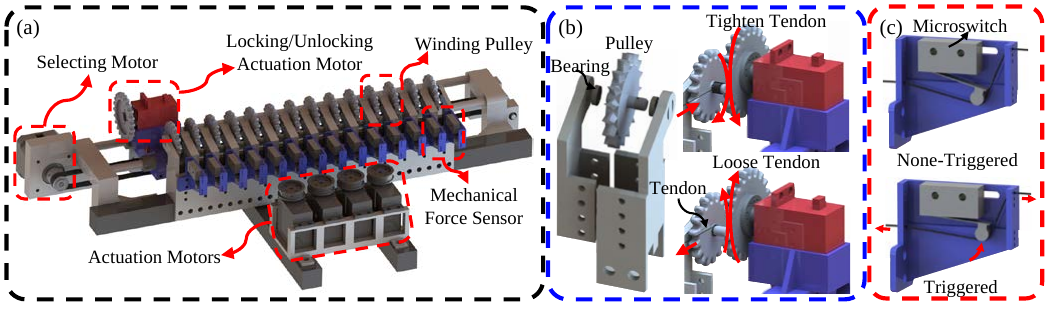}
    \caption{The design of the six-motor actuation module.
    (a) The overview of the actuation module.
    Four motors are placed in the front of the module, controlling the robot to bend.
    Two motors are placed behind, selecting and switching the states of the targeted joints. 
    The modular tendon-winding pulleys are arranged in rows, with their number adjustable to match the joint count of the RTR.
    (b) The working principle of the winding pulley. 
    A tendon is winded on the shaft of the pulley.
    Each pulley has teeth that can mesh with the gears on the locking/unlocking actuation motor.
    When the motor moves into alignment and engages with the target pulley, its rotation tightens or loosens the tendon, thereby enabling the locking or unlocking of the front joint.
    (c) Working principle of the mechanical force sensor. 
    A microswitch is embedded within the sensor, with the tendon passing through it. 
    When the tendon experiences sufficient tensile force, the switch is triggered, closing the circuit and generating a signal.}
    \label{fig:actuationmodule}
\end{figure*}

The conceptual diagram of the RTR is shown in Fig. \ref{fig:robot} (a), which is composed of lockable joints alternatively, enabling the RTR to move spatially.
A set of four driving tendons goes through all joints and mounts on the distal end.
Three parts compose each lockable joint, including a base link, an asymmetric trigger, and a slider, as shown in Fig.\ref{fig:robot}(b) and (c). 
There are tendons attached to both sides of the trigger, which are used to drive the trigger to rotate around the shaft in the opposite direction.

The movement of the slider is restricted by four columns and can only move up and down along the central axis of the base link.
The head of the base link has teeth that can mesh with those on the bottom of the slider.
When the asymmetric trigger rotates and pushes the slider to go down, the slider meshes with the head of the below base link, then the joint angle between the two links is locked.
When the joint is locked, the bottom of the trigger is in surface contact with the top of the slider, and the direction of the force on this plane will pass through the rotation axis of the trigger, forming a dead-point configuration. 
Therefore, the locking of the joint will be quite stable within the tolerance of the material.
Magnets are implanted in both the slider and the base link. When the trigger is rotated to release the dead-point configuration, the slider disengages from the below link under the influence of the magnetic force, and then the joint in unlocked.

The design of the proposed six-motor actuation pack can be seen in Fig. \ref{fig:actuationmodule}. 
The six motors in the pack enable the robot to bend spatially while simultaneously selecting and triggering the locking or unlocking of the targeted joints.
Four motors on the front of the pack are actuation motors, which are used to drive the robot to bend.
Behind the actuation module is a row of winding pulleys, a belt-driven rail, and two motors.
A tendon is wrapped around each pulley and connected to one side of the trigger in a lockable joint, used for actuating the front joint to lock or unlock.
One motor drives the belt-driven rail to align the slider with the targeted pulley, while the other is mounted on the slider to rotate the pulley to wind or release the tendons, as shown in Fig \ref{fig:actuationmodule}.
In the front of each pulley is a force sensor made by a mechanical microswitch.
When the tendon winding on the pulley successfully locks or unlocks the front joint, further retracting the tendon will cause the microswitch to close. 
The signal sent from the switch activates the motor to rotate in the opposite direction, allowing the tendon to loosen and ensuring it will not couple with the subsequent movements of the RTR.


\subsection{Motion Strategy}
Due to the dead-point configuration formed during mechanical locking, it remains locked without the need for external energy input to maintain its state.
Therefore, the actuation module employs a time-phased strategy to sequentially unlock the target joints at each step of the movement, and then the free joints rotate under the actuation of the internal driving tendons.
During movement, the locked joints can be considered as rigid bodies connected to the free joints, which do not interfere with the robot's control.
After completing the current movement step, the actuation module sequentially locks the free joints while the driving tendons are kept tight.
Since the force required for locking/unlocking is relatively small compared to the tension of the driving tendons, the impact of switching joint states on the robot's posture can be considered negligible.
Repeating the above steps, the RTR gradually reconfigures itself and transforms its posture toward the target shape, driven by only six motors throughout the entire movement process.


\section{Modeling}
In our previous work \cite{btLinRobio}, RTR's static model with all joints locked was derived.
In this section, we extend this model to a more general form.  
Based on the rigid body assumption of all links, this general model can be used to calculate the robot’s configuration with external forces under any joint locking state.

\subsection{Static Modeling}
The joints of the RTR can be classified into distal joints and intermediate joints, which have different force analysis considerations during the robot's static analysis, as shown in Fig. \ref{fig:jointmechanics}.  

\subsubsection{The Distal Joint}
\begin{figure}[ht]
    \centering
    \includegraphics[width=0.45\textwidth]{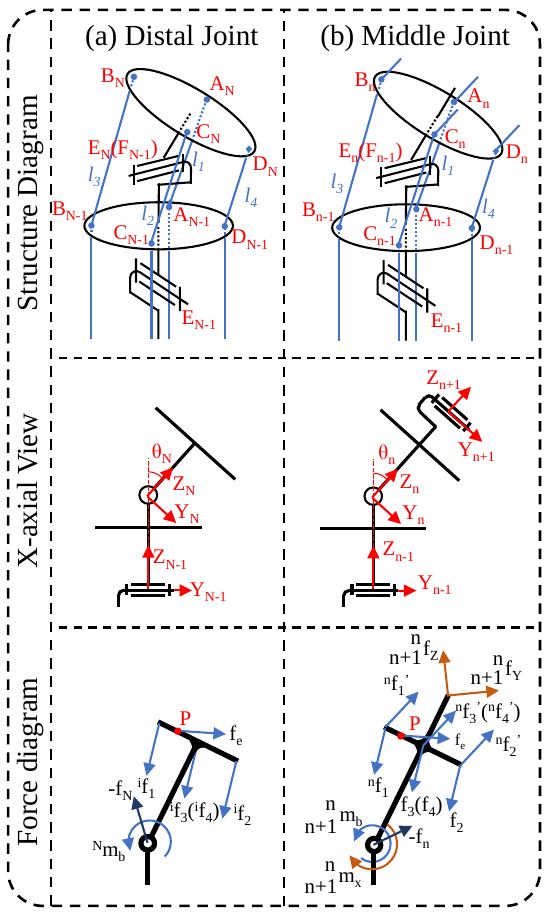}
    \caption{Static analysis of RTR's joints.
    (a) The structure diagram, X-axial view, and the force diagram of the distal joint of the robot. 
    (b) The structure diagram, X-axial view, and the force diagram of any intermediate joint of the robot.}
    \label{fig:jointmechanics} 
\end{figure}
Let the $N^{th}$ joint be the distal one.
As shown in Fig. \ref{fig:jointmechanics}(a), A, B, C, and D are the holes for tendons passing through, and E and F are points representing the upper and lower joints, respectively.
Four driving tendons $l_1, l_2, l_3, l_4$, and local coordinates are also shown, and the positive direction of the X, Y, and Z axis of each local coordinate are from D to B, A to C, and F to E, respectively.
Forces generated by tendons on the $N^{th}$ link are $\mathbf{f}_1$, $\mathbf{f}_2$, $\mathbf{f}_3$, $\mathbf{f}_4$.
The external force $\mathbf{f}_e$ acts on point P.
The resilient torque of the backbone in the $N^{th}$ joint is $^N\mathbf{m}_{b}$, which can be regarded as the torque generated by a torsional spring whose equivalent torsional constant is $K_N$.
It is noted that $K_N$ is a variable that depends on the locking state of the joint.

The $N^{th}$ joint rotates around X-axis, and its rotation angle $\theta_N$ can be derived as follow:
 Torque in the $N^{th}$ joint can be written as
\begin{align}
    &{}^N\mathbf{m} = \widehat{\overrightarrow{FA}}{}^{N}\mathbf{f}_1 + \widehat{\overrightarrow{FC}}{}^{N}\mathbf{f}_2 + \widehat{\overrightarrow{FB}}{}^{N}\mathbf{f}_3 + \widehat{\overrightarrow{FD}}{}^{N}\mathbf{f}_4 + \widehat{\overrightarrow{FP}}{}^{N}\mathbf{f}_e\,,
\end{align}
where $\widehat{\mathbf{x}}$ denotes a skew-symmetric matrix from vector $\mathbf{x}$.

The joint angle $\theta_N$ is depending on the component of ${}^N{\mathbf{m}}$ on the X-axis, which can be written as 
\begin{equation}
    -{}^N{\mathbf{m}_X} = {}^N\mathbf{m}_{b}=\mathbf{\theta}_N K_N \,.
\end{equation}

The resultant force vector in the $N^{th}$ local frame can be written as
\begin{align}
    {}^{N}\mathbf{f} = & 
    {}^{N}\mathbf{f}_1 + 
    {}^{N}\mathbf{f}_2 + 
    {}^{N}\mathbf{f}_3 +
    {}^{N}\mathbf{f}_4 + 
    {}^{N}\mathbf{f}_e\,,
\end{align}
These forces and torques will propagate to the intermediate joints, therefore, they should be expressed in the $(N-1)^{th}$ local frame as follows
\begin{align}
    & {}^{N-1}_N\mathbf{f} = {}^{N-1}_N\mathbf{R}{}^N\mathbf{f}\,, \\
    & {}^{N-1}_N\mathbf{m} = {}^{N-1}_N\mathbf{R}{}^N\mathbf{m}\,, \\
    &{}^{N-1}_{N}\mathbf{R}=Rotx(\theta_N)\,.
    \label{eq:matrix}
\end{align}
where $Rotx(\theta_N)$ denotes the rotation matrix around the X-axis with the angle of $\theta_N$. 
If the joint rotates around the Y-axis, the rotation matrix about $\theta_N$ can be written as   
\begin{equation}
    {}^{N-1}_N\mathbf{R} = Roty(\theta_N)\,.
    \label{eq:matrix2}
\end{equation}

\subsubsection{The Middle Joints}
Assuming the discussed joint in this part is the $n^{th}$ joint, it rotates around the X-axis.
$\theta_{n+1}$, $^{n}_{n+1}\mathbf{f}$, and ${}^{n}_{n+1}\mathbf{m}$ are obtained in the previous calculation.
The structure schematic and free body diagram of the $n^{th}$ joint are shown in Fig. \ref{fig:jointmechanics}(b).
Each tendon goes through the holes of the $n^{th}$ link and has a point contact with the link.
The contact forces $\mathbf{f}_{c1}$, $\mathbf{f}_{c2}$, $\mathbf{f}_{c3}$, $\mathbf{f}_{c4}$ can be calculated as 
\begin{align}
    \mathbf{f}_{ci} = {}^{n}\mathbf{f}_{i} + {}^{n}\mathbf{f}_{i}'\quad\left(i = 1,2,3,4\right)\,.
\end{align}
The resultant forces ${}^{n}\mathbf{f}$ and the component of torque ${}^{n}\mathbf{m}_X$ on the X-axis can be written as
\begin{align}
    {}^{n}\mathbf{f} &= \mathbf{f}_{c1}+ \mathbf{f}_{c2} + \mathbf{f}_{c3} + \mathbf{f}_{c4} + ^{n}_{n+1}\mathbf{f} + \mathbf{f}_e \,,\\
    {}^{n}\mathbf{m}_X &= {}^{n}_{n+1}\mathbf{m}_X + \widehat{\overrightarrow{FE}}{}^{n} \left( ^{n}_{n+1}\mathbf{f}_Y + ^{n}_{n+1}\mathbf{f}_Z\right)
    + \left(\widehat{\overrightarrow{FA}}{}^{n}\mathbf{f}_{c1} \right.  \nonumber \\
    & \left. + \widehat{\overrightarrow{FC}}{}^{n}\mathbf{f}_{c2}
    + \widehat{\overrightarrow{FB}}{}^{n}\mathbf{f}_{c3} + \widehat{\overrightarrow{FD}}{}^{n}\mathbf{f}_{c4} +
    \widehat{\overrightarrow{FP}}{}^{n}\mathbf{f}_e \right)_X \,,
\end{align}
where $^n_{n+1}\mathbf{f}_Y$ and $^n_{n+1}\mathbf{f}_Z$ are the components of $^{n}_{n+1}\mathbf{f}$ on Y- and Z-axis, respectively.
The joint angle $\theta_{n}$ can be calculated as follow
\begin{align}
 -{}^{n}\mathbf{m}_X = {}^{n}\mathbf{m}_b = \mathbf{\theta}_{n}K_{n}  \,.
\end{align}

Utilized the above method to sequentially calculate the rotation angles from the distal joint to the proximal joint, the configuration of RTR can be obtained.

\begin{table}[t]
\centering 
\caption{Verification: RTR follows the constant curvature model in ideal situation}
\begin{tabular}{cccccc}
    \toprule
    & Case 1 & & & Case 2 & \\
    \cmidrule(lr){1-3} \cmidrule(lr){4-6}
    Joint No. & State &Angle ($^\circ$) & Joint No. & State &Angle ($^\circ$)\\
    \cmidrule(lr){1-3} \cmidrule(lr){4-6}
    7 & Free & -19.43 & 7 & Free & -19.43 \\
    6 & Locked & 0 & 6 & Locked & 0 \\
    5 & Free & -19.43 & 5 & Locked & 16 \\
    4 & Locked & 0 & 4 & Locked & 0 \\
    3 & Locked & 16 & 3 & Locked & 16 \\
    2 & Locked & 0 & 2 & Locked & 0 \\
    1 & Free & -19.43 & 1 & Free & -19.43 \\
    \bottomrule
\end{tabular}
\label{tab:idealRTRsimulation}
\end{table}

In ideal situations, where the friction and payload are ignored, the simulation results show that the movement of spatial RTR follows the constant curvature model.
Two example cases are shown in Table \ref{tab:idealRTRsimulation}.
It can be seen that free joint angles remain the same during the movement.
While some middle joints are locked, the rest free joints in this segment have the same rotation angles during the following movements.
This result laid the foundation for the subsequent derivation of the RTR's forward kinematics.

\begin{figure}[t]
\centering
\includegraphics[width = 0.4\textwidth]{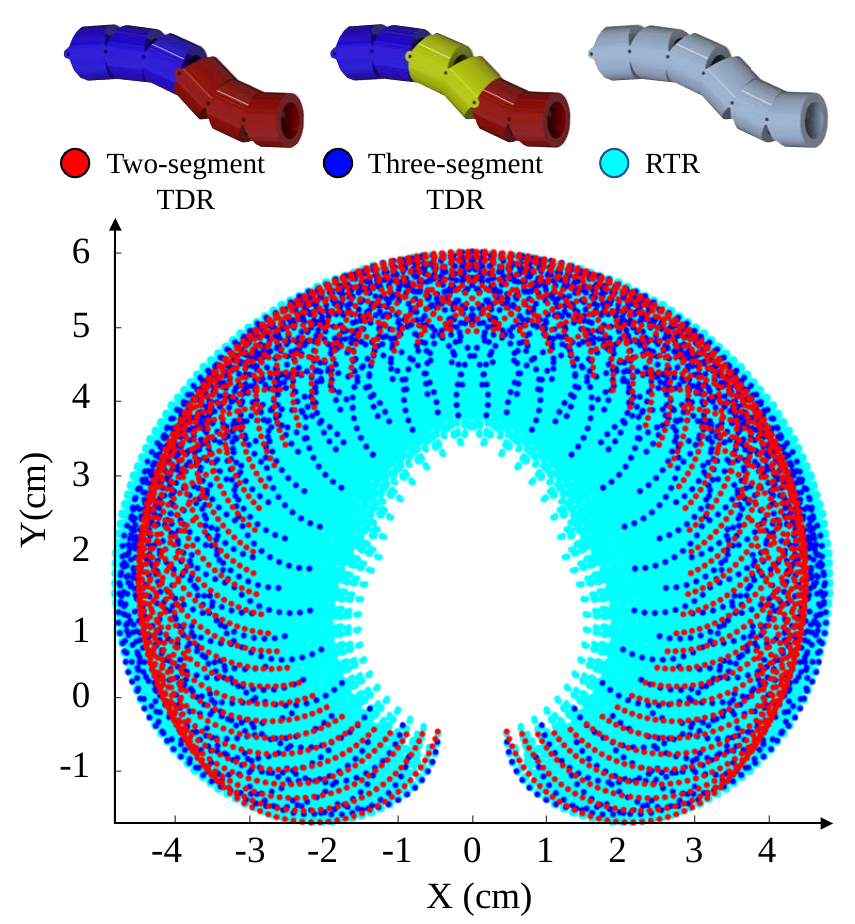}
\caption{workspace comparison between the two-segment TDR, three-segment TDR, and the RTR.
All robots are planar and consist of six joints.
The red, blue, and mint green points denote the reachable positions of the two-segment TDR, three-segment TDR, and the RTR. }
\label{fig:workspace}
\end{figure}
\subsection{Kinematic Modeling}
The kinematics of RTR, when no joints are locked, is identical to that of a traditional TDR.
In our previous work \cite{botao2022lockable}, we demonstrated that the movable part of the planar RTR can be approximated using the constant curvature model, which is commonly used to approximate the shape of TDRs \cite{dalvand2018analytical}.
When the RTR is expanded from a planar to a spatial configuration, it can bend into two perpendicular planes with four driving tendons.
The joints can be divided into two groups based on the rotation direction of their axes.
The driving tendons can also be divided into two groups, $l_1$ and $l_2$ form one group, while $l_3$ and $l_4$ form another.
Based on the derivation of the static model, the angle changes of each group of joints in the spatial RTR under ideal conditions can still be estimated by the constant curvature model.

Fig. \ref{fig:jointmechanics}(b) shows the geometry diagram of the intermediate joint.
When one joint angle is $\theta$, the length of the driving tendons corresponding to this joint can be written as
{\small 
\begin{align}
    &l_1 = \sqrt{(d_3+d_1\cos{\theta}+r\sin{\theta})^2 + (d_1\sin{\theta}-r\cos{\theta}+r)^2}\,,\\
    &l_2 = \sqrt{(d_3+d_1\cos{\theta}-r\sin{\theta})^2 + (d_1\sin{\theta}+r\cos{\theta}-r)^2}\,, \label{eq:l1}\\
    &l_3 = l_4 = \sqrt{d_1^2+d_3^2+2d_1d_3\cos{\theta}}\,.
\end{align}}
As the inverse problem, when $l_2$ is known, the $\theta$ can be found by follows
\begin{align}
    &\theta = \phi - \arcsin{\frac{l_2^2-m}{a}}\,, \label{eq:theta}\\
    &\phi = \arcsin{\frac{2d_1d_3-2r^2}{a}}\,,\\
    &m = d_1^2 + d_3^2 + 2r^2\,,\\
    &a = \sqrt{4d_1^2d_3^2+4r^4+4r^2d_3^2+4r^2d_1^2} \label{eq:a}\,.
\end{align}
    
During the movement, the number of locked joints and their locked angles can be obtained from the previous motion sequence, and angles of free joints can be obtained by the constant curvature model.
Then the posture of RTR can be obtained, and the $n^{th}$ joint's transformation matrix can be written as
\begin{align}
    {}_0^n\mathbf{T} = \prod_{i=1}^n \mathbf{T}(\theta_i) = 
    \begin{bmatrix}
        \mathbf{R}_n & \mathbf{p}_n \\
        \mathbf{0} & 1 
    \end{bmatrix}\,,
\end{align}
where $\mathbf{R}_n$ and $\mathbf{p}_n$ represent the orientation and position of the $n^{th}$ joint, respectively.

\subsection{Reachable Workspace}
With lockable joints, the RTR can reconfigure itself, dividing the joints into segments of varying lengths and shapes, thereby enabling different workspaces. 
Consequently, the whole workspace of RTR has a dramatic increase compared to the traditional TDRs with the same actuation DoFs.

\begin{table}[t]  
	\centering 
	\caption{The maximum $D_p$ of different RTR configurations}
	\begin{tabular}{ccccc}
		\toprule
		Number of division &3  & 4  & 5  & 6 \\
		\midrule
		Maximum $D_p$ (\%)&21.85 & 55.72 & 61.50 & 66.75\\
		\bottomrule
	\end{tabular}
	\label{tab:serviceregion}
\end{table}
\begin{figure*}[t]
\centering
	\includegraphics[width=0.8\textwidth]{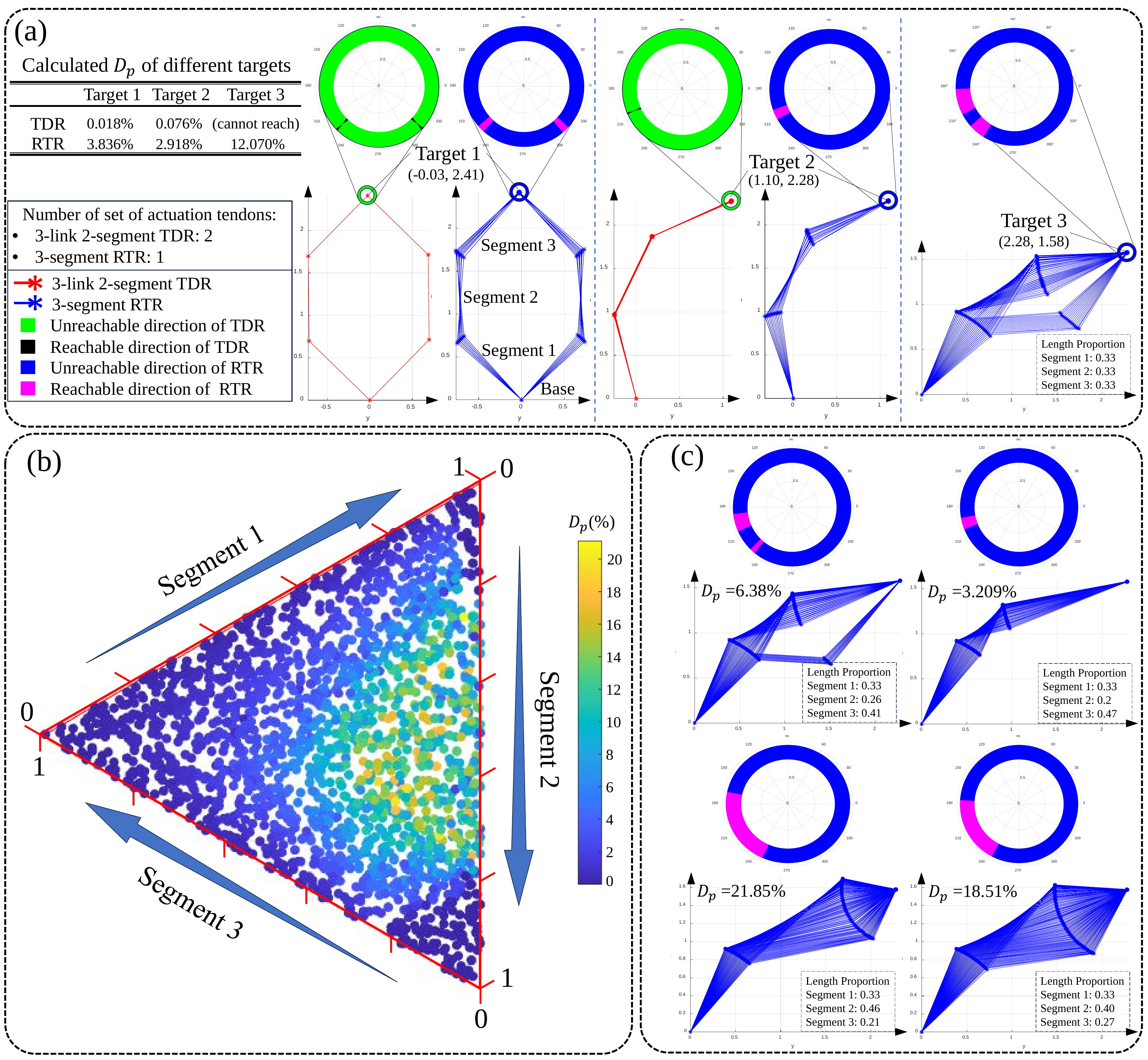}
	\caption{Dexterity analyses of RTR. 
		(a) Comparisons of dexterity of three different targets between 3-link 2-segment TDR and 3-division 1-segment RTR.
        (b) The dexterity map of 3-division 1-segment RTR.
        In the map, each point represents the dexterity of one robot's configuration that reaches the target 3, and the point's coordinates are the ratios of each division length to the whole length.
        (c) 
        Four exemplary points from the dexterity map.
  }
	\label{fig:serviceringsim}
\end{figure*}

Fig. \ref{fig:workspace} shows the workspaces of a two-segment TDR, a three-segment TDR, and an RTR.
The length of each link of the simulated robots is set to $1 cm$, and the angular range of each joint is set to [$-54^\circ,54^\circ$]. 
The simulation results indicate that the workspace of the three-segment TDR encompasses that of the two-segment TDR, while the RTR exhibits a larger workspace than both.
Theoretically, the RTR can achieve the largest workspace among the TDRs with the same structures.
Herein, a lemma is introduced to show the characteristics of the workspace of RTR.
\begin{mylemma}
     With the same structure and actuation module, RTR has a larger reachable workspace than that of the traditional TDR.
\end{mylemma}
\begin{proof}
Without loss of generality, the body of an RTR or a traditional TDR can be treated as a $N$-joint serial manipulator with $N$ movable links of length of $L$.
Let a reference point $\mathbf{p} \in \mathbb{R}^3$ be the distal end, the reachable workspace is defined as the set of all reachable positions of the reference point.

Because each joint of the RTR is individually lockable, when only the final joint $N$ is rotating, the workspace of $\mathbf{p}$ about axis $N$, denoting as $\mathcal{W}_{N}$, is a point set distributing on an arc $\mathcal{P}$, which can be written as 
\begin{equation}
     \mathcal{W}_{N} = \mathcal{P} = \left\{\mathbf{p}\left\vert 
     \mathbf{p}=
     \begin{bmatrix}
         L\cos\theta_N \\
         L\sin\theta_N \\
         0
     \end{bmatrix}\right.,
      \theta_N \in [\theta_{min}, \theta_{max}]
     \right\},
\end{equation}
where $\left[\theta_{min}, \theta_{max}\right]$ denotes the motion range of each joint.

When extending to the $(N-1)^{th}$ joints, the new workspace $\mathcal{W}_{N-1}$ can be obtained by rotating $\mathcal{W}_{N}$ around axis $N-1$, which is the union of all possible point sets after a rigid motion defined by $\theta_{N-1}$: 
\begin{equation}
  \mathcal{W}_{N-1} = \bigcup_{\theta_{N-1}} {}^{N-2}_{N-1}\mathbf{T}(\theta_{N-1}) \; \mathcal{W}_{N} \,,
\end{equation}
where ${}^{N-2}_{N-1}\mathbf{T}$ is the homogeneous transformation matrix of $N-1$ frame with respect to $N-2$ frame. 
Hereafter, we have implicitly enhanced each point to its homogeneous coordinate.

Sequentially, workspace about the first joint can be obtained iteratively, which equals the complete workspace of RTR as follows 
\begin{align}
    &\mathcal{W}_{RTR} = \mathcal{W}_{1} = \bigcup_{\theta_{1}} {}^{0}_{1}\mathbf{T}(\theta_{1}) \; \mathcal{W}_{2} \nonumber \\
    &= \bigcup_{\theta_{1}} {}^{0}_{1}\mathbf{T}(\theta_{1})\left( \bigcup_{\theta_{2}} {}^{1}_{2}\mathbf{T}(\theta_{2}) \; \mathcal{W}_3\right) \nonumber \\
    & \qquad \vdots \nonumber \\
    &= \bigcup_{\theta_{1}} {}^{0}_{1}\mathbf{T}(\theta_{1})  \left(\bigcup_{\theta_{2}} {}^{1}_{2}\mathbf{T}(\theta_{2}) \cdots
    \left(\bigcup_{\theta_{N-1}} {}^{N-2}_{N-1}\mathbf{T}(\theta_{N-1}) \; \mathcal{W}_{N}\right) \right) \nonumber\\
    &= \bigcup_{\theta_{1},\cdots,\theta_{N-1}} {}^{0}_{1}\mathbf{T}(\theta_{1}) \cdots {}^{N-2}_{N-1}\mathbf{T}(\theta_{N-1}) \; \mathcal{P}.
\end{align}

On the other hand, according to the constant curvature model, all joints in a traditional one-segment TDR have the same rotation angles ideally, and the workspace of the TDR can be written as 
\begin{align}
    &\mathcal{W}_{TDR} = {}^{0}_{1}\mathbf{T}(\theta) \cdots {}^{N-2}_{N-1}\mathbf{T}(\theta) \; \mathcal{P}^{\prime} \\
    &\mathcal{P}^{\prime}  = \left\{\mathbf{p}\left\vert 
     \mathbf{p}=
     \begin{bmatrix}
         L\cos\theta \\
         L\sin\theta \\
         0
     \end{bmatrix}\right.,
      \theta \in [\theta_{min}, \theta_{max}]
     \right\}.
\end{align}

Consequently, $\mathcal{W}_{TDR}$ is a subset of $\mathcal{W}_{RTR}$ which can be written as
\begin{align}
    &\mathcal{W}_{TDR} \subsetneqq \mathcal{W}_{RTR}.
\end{align}
Therefore, the RTR has a larger workspace than the traditional TDR.
It is worth noting that the workspace of the traditional TDR is only a special case of that of the RTR.

\end{proof}

\subsection{Dexterity Analysis}
To analyze the dexterity of RTR for any reachable point within the workspace, the service ball is used as an index of dexterity \cite{wu2016dexterity,badescu2004new}.
The robot can reach a point with multiple configurations, and the tangent lines of the end-effector intersect the service ball with multiple points, collectively forming a region noted as the service region.
Set one of the reachable points as $Q$, and the dexterity of spatial RTR about it can be defined as
\begin{align}
	D_s(Q) = \frac{A_R(Q)}{A_S} \times 100\%\,,
\end{align}
where $A_R(Q)$ is the area of service region, and the $A_S$ is the total surface area of the service ball.
For a planar RTR, the workspace is on a plane, and its service ball and region degrade to a ring and an arc, respectively.
The dexterity of the planar RTR about point $Q$ can be written as
\begin{align}
	D_p(Q) = \frac{\theta_R(Q)}{2\pi}\times 100\%\,,
\end{align}
where $\theta_R(Q)$ is the angle corresponding to the service arc.
Because service arc can show dexterity more intuitively, planar TDR and RTR are utilized in subsequent analyses.
In addition to comparing the dexterity between the RTR and the traditional TDR, the dexterity is also analyzed from two aspects: the length and number of robot segments.

\subsubsection{Dexterity Comparison} 
Simulations on the dexterity of a three-link two-segment TDR and a three-segment RTR around the same position have been carried out, the results can be seen in Fig. \ref{fig:serviceringsim}(a).
For the same targets, RTR can reach them by more postures in different directions, which corresponds to larger service arcs and indicates better dexterity.

\subsubsection{Effect of Segment Length} 
In addition, RTR can reconfigure the lengths of the robot segments to change the dexterity. 
Take the target point 3 shown in Fig. \ref{fig:serviceringsim}(a) as an example, after varying the segment lengths of three-segment RTR and calculating the corresponding $D_p$, a dexterity map can be drawn to show the relationship between the $D_p$ and the length of each segment, as shown in Fig. \ref{fig:serviceringsim}(b).
Here, the variable segment length is relative to the whole length of the robot, 2000 combinations of division lengths have been sampled, and the corresponding $D_p$ has been calculated.
In the map, a point appearing more yellow indicates that the RTR, with the corresponding combination of segments, can reach the target from a greater number of directions, whereas a more blue hue indicates fewer possible approach directions.
The map shows that about target 3, RTR can obtain the maximum dexterity when the proportion of the three segments is approximately 0.33:0.47:0.2. 
In Fig. \ref{fig:serviceringsim}(c), four examples of RTRs with different segment combinations are shown, and their $D_p$s about the target 3 varies from 3.209\% to 21.85\%. 

\subsubsection{Effect of Segment Number} 
Changing the number of robot segments by reconfiguration can also change the dexterity of RTR.
Continuing with the target 3 mentioned above, changing the number of robot segments from 3 to 6 while the total length of RTR remains unchanged, the corresponding maximum $D_p$ can be calculated. 
The results are shown in Table \ref{tab:serviceregion}, the maximum $D_p$ of RTR can be promoted from 21.85\% to 55.72\%, 61.50\%, and 66.75\% by adding the number of segments from 3 to 4, 5, 6, respectively.
It demonstrates that the reconfiguration capability of RTR has a significant impact on its dexterity, allowing the same RTR to present various working performances.
 
\section{Experiment and Demonstration}
\subsection{Static Model Validation}
\begin{figure}[t]
    \centering
    \includegraphics[width=0.42\textwidth]{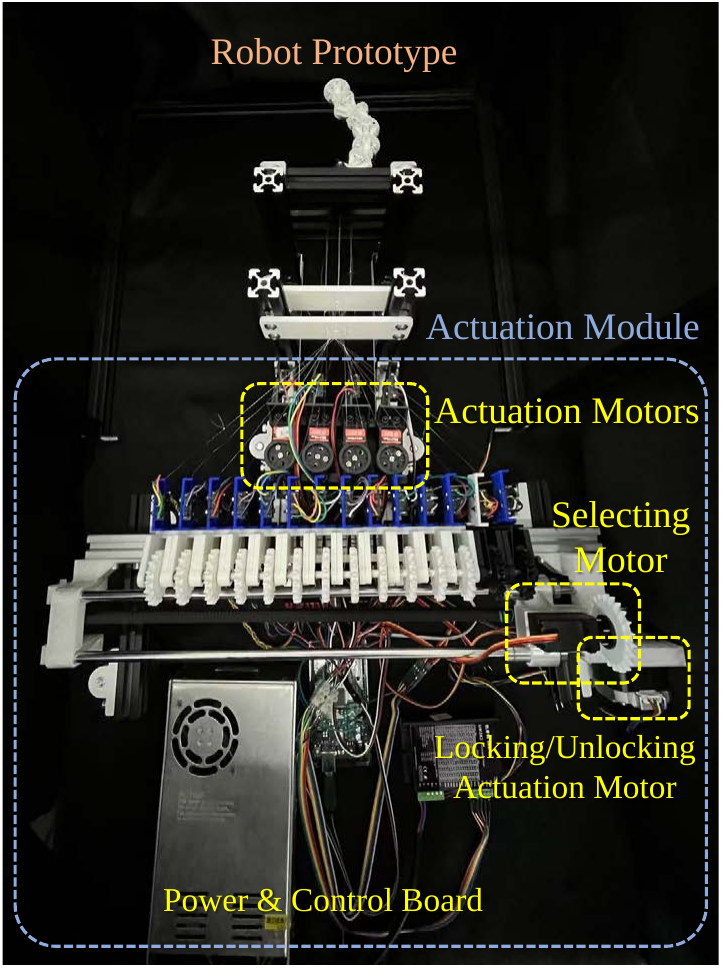}
    \caption{Overview of the seven-joint RTR system.
    The system consists of the robot prototype, the actuation module, and the power module.}
    \label{fig:robotsystem}
\end{figure}
\begin{figure}
    \centering
    \includegraphics[width=0.4\textwidth]{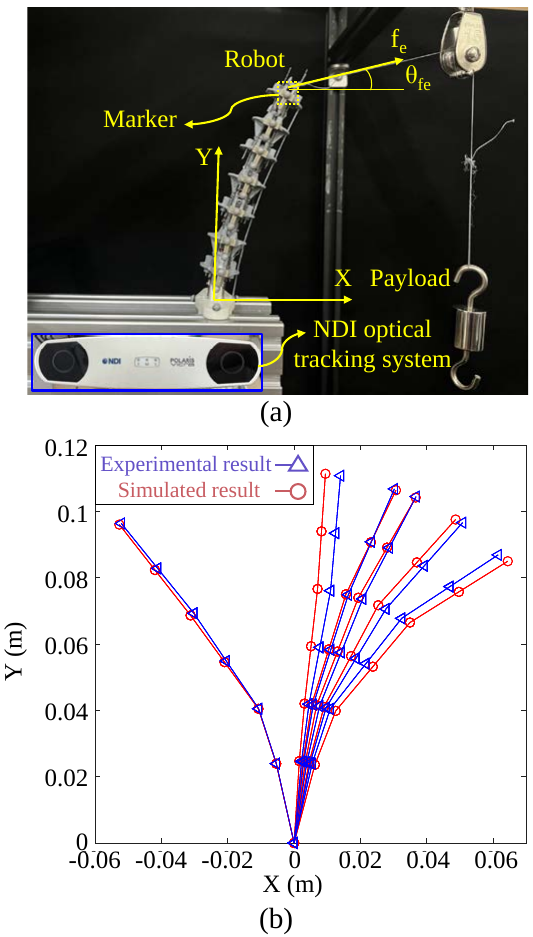}
    \caption{Experimental validation of the proposed static model of RTR.
    (a) Experimental setup. 
    The payload connects to the tip of the RTR prototype by cables and a fixed pulley.
    $\theta_{\mathbf{f}_e}$ is angle between the cable and x-axis.
    Markers are placed on each link, and an optical tracking system (NDI Polaris) is used to record the position data.			
    (b) The comparison between simulation and experimental data.}
    \label{fig:staticsexp}
\end{figure}
The proposed static model of RTR is validated by experiments in this section, and the experimental setup is shown in Fig. \ref{fig:staticsexp} (a).
To enhance the consistency of the static model with real-world situations, the frictional force between the driving tendon and link can be calculated in an exponential form \cite{huang2018statics} as
    \begin{align}
        ^{n}\mathbf{f} =& ^{n+1}\mathbf{f}\exp{\left(\mu \theta_{n+1}\right)}\,,
    \end{align}
where $\mu$ is the coefficient of friction.

\begin{table}[b]
	\centering 
	\caption{Error between simulation and experiment of static model}
	\begin{tabular}{cccccc}
		\toprule
		 & & State & &Joint Angle Error ($^\circ$)\\
        \cmidrule(lr){2-4}
		Locked Joints & $\mathbf{f}_1$ (N) & $\mathbf{f}_2 (N)$ &$\mathbf{f}_{ex} $ (N) &(mean$\pm$std)\\
		\midrule
        None &$0$ &$0.98$ &$0$ & 0.5105 $\pm$ 7.057\\
        None &$0$ &$1.96 $ &$0$& 1.170 $\pm$ 6.790\\
		None &$ 0$ &$0$ &$0.0196$& 0.1781 $\pm$ 5.973\\
        None &$0.98$& $ 0 $ &$0.0196$& 0.4135 $\pm$ 13.15\\
        1,2,3,4 &$0.98$ &$0 $ &$0.0196$& 1.038 $\pm$ 7.151\\
        1,2,3,4 &$0$ &$0.98$ &$0.0196$& 0.4214 $\pm$ 6.971\\
		\bottomrule
	\end{tabular}
	\label{tab:ExpStatics}
\end{table}

After exerting a payload on the distal end, the prototype begins to deform and the direction of the external force will vary from the original one.
However, in the model derivation, the direction of external force is fixed relative to the local frame of the distal end, which is contradictory to the actual situation.
Therefore, an optimization-based method has been proposed to find the steady state direction of the external force $\theta_{\mathbf{f}_e}$ and robot posture under external payload. 
The objective function $E_\theta$ is defined as the difference between the direction of external force and $\theta_{\mathbf{f}_e}$.
After setting an initial direction of the external force, the steady state posture of the robot can be found iteratively by minimizing the $E_\theta$.

Simulation and experimental results are compared in Table \ref{tab:ExpStatics}, with different locking strategies, pulling forces ($\mathbf{f}_1,\mathbf{f}_2$), and external forces ($\mathbf{f}_{ex}$). 
The calibrated $\mu$ was found to be 0.085 by using the experimental data from the robot without any locked joint or external force.
Posture errors between the experimental and simulated data are shown in Fig. \ref{fig:staticsexp} (b) and Table \ref{tab:ExpStatics}.
The results indicate that the model is accurate with the worst mean error and its standard deviation of all joints less than 0.064 and 0.035 degrees, respectively.

\subsection{RTR Working Performance Demonstration}
\begin{figure*}
    \centering
    \includegraphics[width=0.95\textwidth]{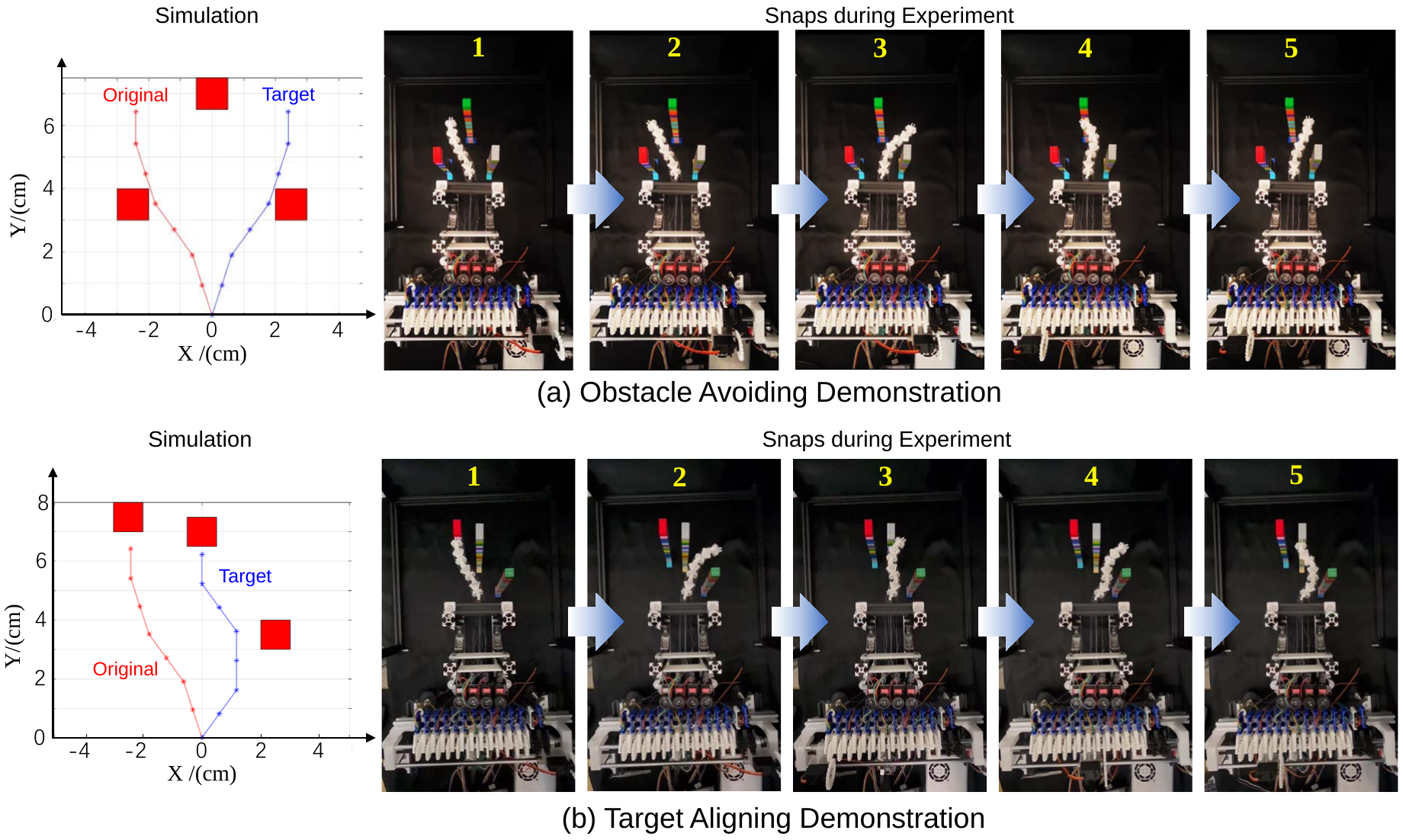}
    \caption{Demonstration of the RTR prototype working in the complex environment.
    Simulated trajectories can be seen on the left.
    (a) Obstacle avoidance demonstration.
    The RTR prototype first contracts its body, then swings across the obstacle, and finally extends its body to achieve the target posture.
    (b) Target aligning demonstration.
    The robot continuously reconfigures itself and moves to avoid obstacles and ultimately aligns its end effector with the target.}
    \label{fig:PerfomanceDemo}
\end{figure*}

To validate the capability of the RTR to work in complex environments under the control of the actuation module by leveraging its reconfigurability and time-phased actuation strategy, two verification experiments are designed and conducted.
During the experiments, RTR is required to safely transition from the original pose to the target pose in a space with obstacles.
The simulation and experimental results can be seen in Fig. \ref{fig:PerfomanceDemo}.
Users can reconfigure the RTR in each step according to task requirements, enabling the robot to perform complex motions that traditional TDRs require additional DoF to accomplish.
Through actions such as contraction, swinging, and extension, the RTR can easily perform tasks such as obstacle avoidance and target alignment, demonstrating its feasibility in working within complex environments.

\section{Discussion}
Besides kinematics, workspace, stiffness, and motion planning, there are still many factors that should be discussed before applying RTR to many practical tasks. 
Herein, we discuss other important aspects.

\begin{figure}[t]
\centering
\includegraphics[width = 0.48\textwidth]{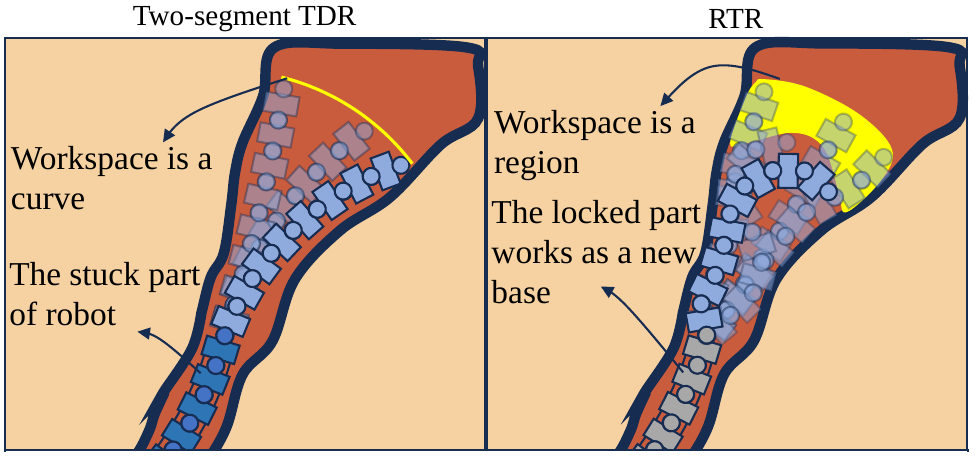}
\caption{Schematic diagram of RTR and underactuated TDR
working in confined space.
RTR can reach some positions that TDR with the same amount of control motors cannot reach.
Although the robot is in a confined space, RTR can still have a region workspace, while TDR just has a curved workspace as some robot segments have been stuck by the torturous environment.} 
\label{fig:workinconfinedspace}
\end{figure}

\subsection{Application in Confined Spaces}
As shown in Fig. \ref{fig:workinconfinedspace}, a typical MIS application scenario is used to demonstrate the superiority of RTR in workspace and dexterity.
During an MIS procedure, clinicians usually require a robot to pass through the complex environment with the end-effector being dexterous.
However, work in confined environments like the intestine, pharynx, or sinus is difficult for TDR \cite{yang2023magnetically,hong2020magnetic}.
Although multi-segment TDR seems to be very dexterous, many segments would be stuck by the tortuous incision path.
When the end-effector finally reaches the lesion, the robot remains only limited DoFs for operation, which results in limited workspace and dexterity.

Thanks to its reconfigurability, in comparison, RTR is immune from the aforementioned problem.
When RTR moves along the same tortuous path to reach the lesion, we can lock the joints that far from the lesion as a new base and actuate the rest, which results in a larger workspace with higher dexterity around the lesion.

\subsection{Scaling}
\begin{figure}[t]
\centering
\includegraphics[width=0.48\textwidth]{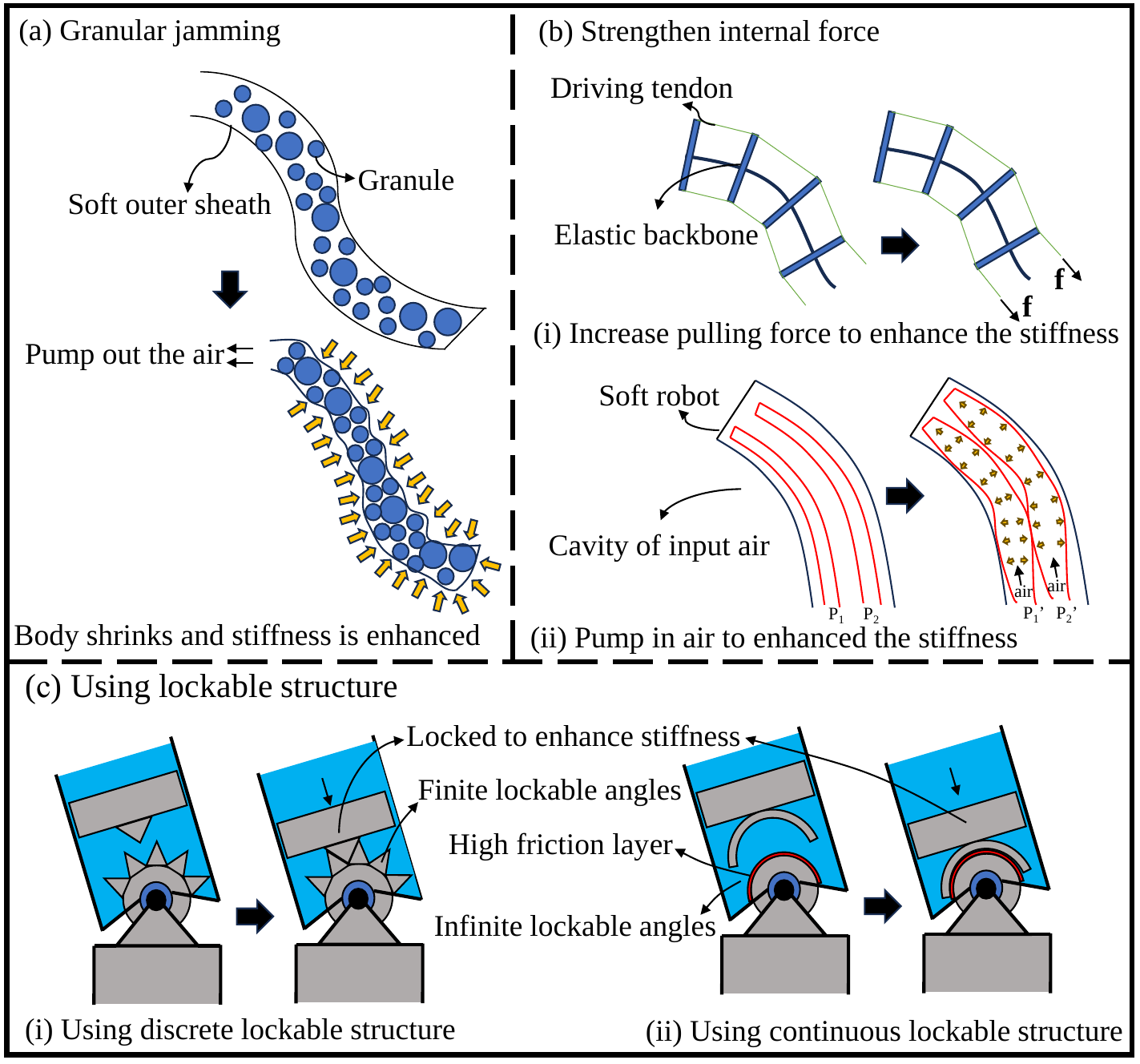}
\caption{Principles of varying stiffness. 
    (a) Robots with granules inside can use a pump to create a vacuum, the jamming of granules would be caused by the atmospheric pressure, and the robot's stiffness is enhanced.
    (b) Robots controlled by driving tendons or inputting fluid or gas can enhance their stiffness by enlarging the internal forces.
    (c) Robots can use lockable structures to lock parts of the robot to enhance their stiffness.
    The lockable structures can be divided into two types:  discrete lockable and continuous lockable.}
\label{fig:introducevariablestiffness}
\end{figure}
Scaling is the ability to scale up or down the size of the manipulator, which will be considered while customizing RTR for specific tasks \cite{li2017kinematic}.
In most situations, the lockable structure cannot be linearly scaled, it may be better to change the type of lockable structure rather than just scaling it.
For small-size RTR, the cross-section cannot contain bulky structures, thereby methods that do not need much space yet offer enough force could be adopted, such as shape memory alloy (SMA), phase change material like wax, etc.
For large-size RTR there is more volume for lockable structures to be implanted.
Methods that offer great locking performance can be chosen, such as using worm gear and brake mechanisms, etc.
Some schematic of lockable structure's principles are shown in Fig. \ref{fig:introducevariablestiffness}.

For RTR's outstanding ability to work in small narrow environments, investigating smaller lockable structures does great help to realize the potential of RTR.
Many materials and structures have the potential to be used in lockable structures, like dielectric material, magnetorheological fluid, origami structures, and so on.

\subsection{Safety}
The safety of RTR is an essential issue that should be paid attention to.
Unexpected failure of locking would cause undesired movements of the robot, which may cause danger to the user and environment.
In addition, the working principles of some lockable structures may affect the environment.
Taking the SMA as an instance, SMA requires heating to deform, which may lead to high temperature and the risk of scald.

To avoid these sudden risks, practitioners should understand the physical and chemical properties of materials used in the lockable structures, and validation experiments of the structure should be carried out.
Furthermore, accurate static and kinematic models should be introduced, which can guide practitioners in manipulating the robots safely.
\section{Conclusion}
This paper presents the design and working principle of RTR, establishes its forward kinematic and static models, and evaluates its advantages in workspace and dexterity compared with traditional TDRs.
A seven-joint RTR prototype has been built, along with a compact actuation module with only six motors.
Leveraging a time-phased actuation strategy and reconfigurability, RTR surpasses traditional TDRs in dexterity while fundamentally eliminating intersegmental coupling during movements, significantly reducing the complexity of modeling and control.
Experiments were conducted to validate the static model of the robot, and the position error between the experiment and simulation was $0.064\pm0.035$ degrees (mean$\pm$standard).
Furthermore, verification experiments demonstrated the feasibility of the RTR prototype in complex working environments.
Through a sequence of contraction, swinging, and stretching motions, RTR successfully performs obstacle avoidance and successfully transitions into the target posture.

However, further research is required to enhance the performance of RTR and better understand its limitations.
In particular, the development of smaller and lighter lockable joint structures, as well as efficient methods for rapid kinetostatic computation, remains a key challenge.
Additionally, since the motion strategy of RTR differs from that of traditional TDRs, dedicated motion planning algorithms must be developed.
Moreover, a deeper understanding of the relationship between joint locking strategies and the robot's stiffness is essential, which will facilitate the application of RTR in more scenarios, such as organ support in MIS and structural reinforcement for debris in disaster sites, etc.
Although research on RTR is still in its early stages, we believe that it will emerge as a highly promising solution for certain applications in the future.
\bibliographystyle{IEEEtran}
\bibliography{ref}

\newpage


\vspace{-11 mm}
\begin{IEEEbiography}[
	{\includegraphics[width=1in,height=1.25in,clip,keepaspectratio]{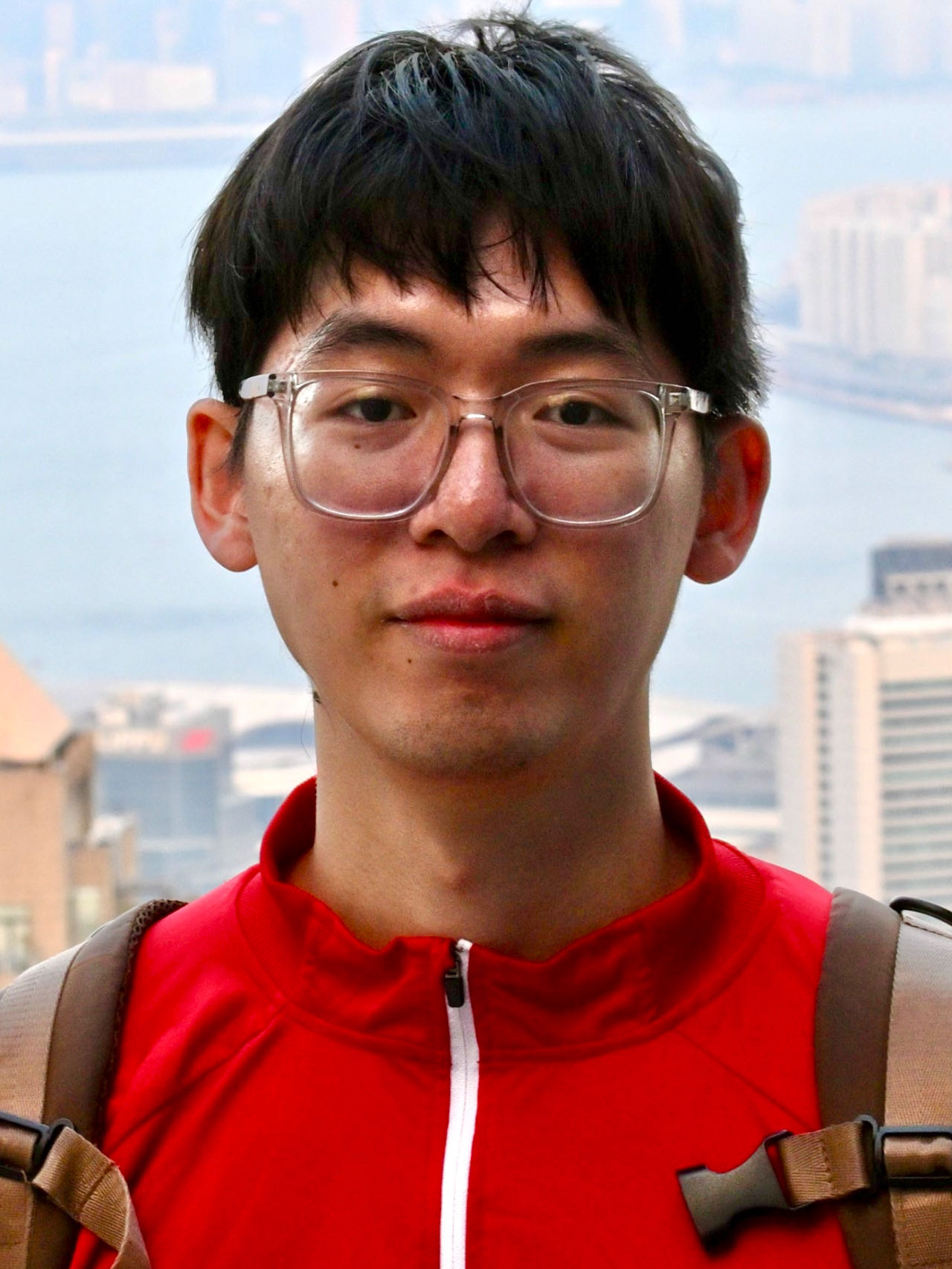}}
	]{Botao Lin} received the B.E. degree in Mechanical Design Manufacturing and Automation and M.E. degree in Mechanical Engineering from the School of Mechanical Engineering and Automation, Harbin Institute of Technology, Shenzhen, in 2022 and 2024, respectively. 
    He is currently pursuing a Ph.D. degree from the Department of Electronic Engineering, The Chinese University of Hong Kong (CUHK), Hong Kong.
    His main research interests include continuum surgical robotics and variable-stiffness robotics.
\end{IEEEbiography}

\vspace{-11 mm}
\begin{IEEEbiography}[
	{\includegraphics[width=1in,height=1.25in,clip,keepaspectratio]{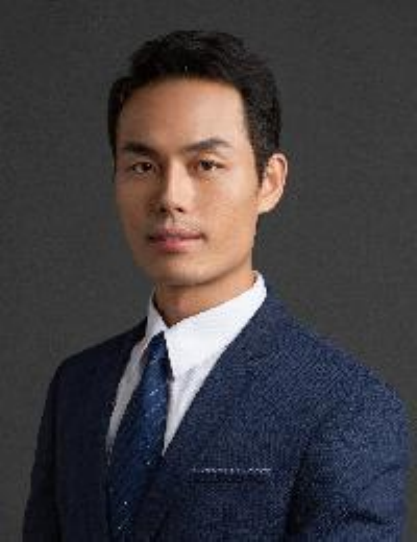}}
	]{Shuang Song} (Member, IEEE) received B.S. degree in Computer Sc. \& Tech. from North Power Electric University, M.S. degrees in Computer Architecture from Chinese Academy of Sciences, and his Ph.D. degree in Computer Application Tech. from University of Chinese Academy of Sciences, China, in 2007, 2010 and 2013, respectively. Now he is an Professor in Harbin Institute of Technology, Shenzhen, China. His main research interests include magnetic tracking and actuation for Bioengineering applications.
\end{IEEEbiography}

\vspace{-11 mm}
\begin{IEEEbiography}[
	{\includegraphics[width=1in,height=1.25in,clip,keepaspectratio]{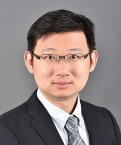}}
	]{Jiaole Wang} (Member, IEEE) received the B.E. degree in mechanical engineering from Beijing Information Science and Technology University, Beijing, China, in 2007, the M.E. degree from the Department of Human and Artificial Intelligent Systems, University of Fukui, Fukui, Japan, in 2010, and the Ph.D. degree from the Department of Electronic Engineering, The Chinese University of Hong Kong (CUHK), Hong Kong, in 2016. He was a Research Fellow with the Pediatric Cardiac Bioengineering Laboratory, Department of Cardiovascular Surgery, Boston Children’s Hospital and Harvard Medical School, Boston, MA, USA. He is currently an Associate Professor with the School of Mechanical Engineering and Automation, Harbin Institute of Technology, Shenzhen, China. His main research interests include medical and surgical robotics, image-guided surgery, human-robot interaction, and magnetic tracking and actuation for biomedical applications.
\end{IEEEbiography}

\vfill

\end{document}